\documentclass{article}

\usepackage[preprint]{neurips_2026}

\usepackage[utf8]{inputenc} 
\usepackage[T1]{fontenc}    
\usepackage{hyperref}       
\usepackage{url}            
\usepackage{booktabs}       
\usepackage{amsfonts}       
\usepackage{amsmath}
\usepackage{nicefrac}       
\usepackage{microtype}      
\usepackage{xcolor}         
\usepackage{graphicx}
\usepackage{color}
\usepackage{tabularray}
\usepackage{amsthm}
\usepackage{enumitem}

\theoremstyle{plain}
\newtheorem{theorem}{Theorem}[section]

\theoremstyle{definition}
\newtheorem{definition}[theorem]{Definition}

\theoremstyle{remark}

\title{Position: Universal Time Series Foundation Models Rest on a Category Error}

\author{%
Xilin Dai\\
ZJU-UIUC Institute\\
Zhejiang University\\
Hangzhou, China\\
  \texttt{xilin2023@zju.edu.cn}
  \And
Wanxu Cai\\
School of Software\\
Tsinghua University\\
Beijing, China\\
  \texttt{caiwx22@mails.tsinghua.edu.cn}
  \And
Zhijian Xu\\
Department of Computer Science and Engineering\\
The Chinese University of Hong Kong\\
Hong Kong, China\\
  \texttt{zjxu21@cse.cuhk.edu.hk}
  \And
Qiang Xu\\
Department of Computer Science and Engineering\\
The Chinese University of Hong Kong\\
Hong Kong, China\\
  \texttt{qxu@cse.cuhk.edu.hk} 
}

\begin{document}

\maketitle

\begin{abstract}
This position paper argues that \textbf{the pursuit of ``Universal Foundation Models for Time Series" rests on a fundamental Category Error, mistaking a structural Container for a semantic Modality.} Because time series domains hold incompatible generative processes, treating them as a unified modality forces models into ill-posed optimization. Consequently, we argue that the model's inability to handle temporal non-stationarity is the direct consequence of this Category Error. Finally, our Call to Action urges the community to adopt a control-theoretic perspective. Evaluation metrics must pivot from zero-shot accuracy to new benchmarks measuring Time-to-Recovery (TTR), which quantifies adaptability after an external intervention. We propose that the emerging Conditional Control Agent paradigm provides a promising pathway to overcome these limits.
\end{abstract}

\section{Introduction: The Universal Prior Fallacy}
The current landscape of machine learning is defined by the quest for universality. Following the transformative success of Large Language Models (LLMs) like GPT-4 \cite{achiam2023gpt} and Vision Transformers (ViTs) \cite{dosovitskiy2020image}, the research community has turned its gaze toward the final frontier of data types: Time Series. The prevailing hypothesis driven by the release of models is that ``scale is all you need." The argument posits that if we aggregate enough temporal data—billions of tokenized points from stock markets, weather stations, intensive care units, and server logs—a sufficiently large Transformer will learn a ``Universal Prior" of temporal dynamics, unlocking emergent forecasting capabilities analogous to the reasoning abilities of LLMs \cite{schmidhuber2024learning}. This aspiration is rooted in the belief that temporal sequence modeling shares a fundamental isomorphism with language modeling, where the prediction of the next token (or time step) serves as a proxy for understanding the underlying generative grammar of the world.

This paper challenges this hypothesis. \textbf{We posit that the pursuit of a Universal Time Series Foundation Model is based on a category error that confuses a Data Type (a container for information) with a Modality (a shared semantic space).} Our critique is not that broad pretraining is useless; it is that broad pretraining without explicit domain, task, or intervention context should not be treated as the time-series analogue of language-model pretraining. In language, the word ``apple" shares a limited semantic grounding whether it appears in a recipe or a poem (a fruit or company). In vision, an ``edge" or a ``texture" obeys the same laws of optics whether it outlines a cat or a car. Time series data do not possess such shared grounding. A ``spike" in a financial ticker (driven by market sentiment and liquidity shocks) has no causal or physical relationship to a ``spike" in an EKG (driven by electrophysiological depolarization) \cite{lee2025cognitive} or a ``spike" in a turbine’s vibration sensor (driven by mechanical friction).

Treating these disparate signals as a unified modality creates an ill-posed optimization problem. When a model is forced to minimize loss across these conflicting physical rules, it does not learn a generative truth; it learns the intersection of priors—the lowest common denominator of statistical properties. We risk computational waste by building what are essentially \textbf{``Generic Filters"}—expensive approximations of simple statistical smoothers—instead of solving the core problem of forecasting: temporal non-stationarity and drift.

The stakes for this error are high. Real-world systems are open and dynamic; they are subject to interventions and regime shifts that a history-only model cannot predict \cite{chen2025foundation, xu2025one}. To articulate this, our position is structured as follows. First, we outline \textbf{Argument I: The Category Error}, which suggests under stated assumptions that the union of time series domains collapses to a trivial symmetry group. Second, we present \textbf{Argument II: Non-Stationarity as the Consequence of Category Error}, formally demonstrating that these generic filters are structurally incapable of adapting to external regime shifts by hitting the Autoregressive Blindness Bound. Finally, we conclude with a \textbf{Call to Action}, urging the community to shift from zero-shot accuracy to Time-to-Recovery (TTR) as the core evaluation metric, and advocating for conditional control systems.

\section{Related Works}
\subsection{Foundation Models for Time Series}
The aspiration to replicate the success of Natural Language Processing (NLP) in time series field has driven a bifurcation in research: \textbf{Native Pre-training} and \textbf{Language Model Adaptation}. Native time series pre-training has progressed through several architectural paradigms. Models such as MOMENT \cite{goswami2024moment} and Moirai \cite{woo2024moirai} adopt encoder-only architectures with masked modeling objectives. TimeGPT\cite{garza2023timegpt}, TimesFM \cite{das2024timesfm}, and Moirai 2.0 \cite{liu2025moirai2.0}, shifted to decoder-only, autoregressive transformers. Chronos \cite{ansari2024chronos} and Chronos-2 \cite{ansari2025chronos} formulate time series forecasting as a language modeling problem via value quantization. More recent efforts explore alternative scaling and modeling strategies, such as Time-MoE \cite{shi2024timemoe}, which employs a Mixture-of-Experts architecture, and Sundial \cite{liu2025sundial}, which adopts a continuous flow-matching formulation.

In parallel, another line of research repurposes off-the-shelf large language models for time series forecasting. LLMTime \cite{gruver2023llmtime} directly inputs the forecasting task into the prompt. GPT4TS \cite{zhou2023gpt4ts} and Time-LLM \cite{jin2023timellm} introduce reprogramming-based approaches that map time series patches into the embedding space of frozen LLMs. Building on autoregressive alignment, AutoTimes \cite{liu2024autotimes} enables in-context forecasting, while UniTime \cite{liu2024unitime} proposes a unified Language–Time Transformer to jointly process temporal data and language prompts. The common hypothesis underlying these works is that ``Scaling Laws" are applicable to time series, and there exists a universal ``Temporal Grammar". However, this assumption is challenged by several works, mainly in two key areas: \textbf{Debate of Model Architectures} and \textbf{Distributional Discrepancy}.

\vspace{-8pt}

\subsection{Debate of Model Architectures}
The trajectory of time series modeling has attempted to mirror the paradigm shift in NLP, though with contested results. In NLP, the transition from RNN \cite{elman1990rnn} and LSTM \cite{hochreiter1997lstm} to the Transformer \cite{vaswani2017transformer} architecture provided substantial gains in capturing long-range dependencies. Subsequent scaling of this architecture established a clear ``scaling law" where increased parameters and data correlated directly with model capabilities.

To address long-sequence time series modeling, researchers initially directly reused the transformer structure. Informer \cite{zhou2021informer} and Autoformer \cite{wu2021autoformer} introduced sparse attention and decomposition-based designs, followed by Fedformer \cite{zhou2022fedformer} and Crossformer \cite{zhang2023crossformer}, which incorporate frequency-domain modeling and cross-dimension attention. \textbf{However, unlike in NLP, the supremacy of Transformers in time series has been challenged.} DLinear \cite{zeng2023are} demonstrated that a simple, single-layer linear model could outperform complex, modification-heavy Transformers. This sparked a resurgence of MLP-based and linear architectures, such as TiDE \cite{das2023tide}, TimeMixer \cite{wang2024timemixer}, and TimePrism \cite{dai_samples_2025}. Some works also explored the underlying causes and pointed out the ``Attention Collapse" phenomenon \cite{dong2023stabilizing}, which indicates that for temporal data, the heavy machinery of the Transformer often degrades into an expensive implementation of a low-pass filter or a seasonal moving average, lacking the emergent reasoning capabilities.

\subsection{Distributional Discrepancy}
\textbf{Apart from the disagreement on the model architecture, the more significant issue that time series faces is the issue of generalization.} Unlike NLP where semantic tokens maintain consistent meaning across contexts, time series analysis is plagued by a fundamental violation of the i.i.d. assumption. This distributional discrepancy manifests along two orthogonal axes—\textbf{temporal and spatial}—both of which undermine the premise of a unified FM. The temporal non-stationarity indicates that the training distribution differs from the future distribution. In the real world, the simplest form of distribution change is the variation in the statistical values of a sequence. The most common types are changes in the mean and variance. Methods like RevIN \cite{kim2021revin} and Non-stationary Transformers \cite{liu2022nonstationtrans} reduced this problem. However, apart from the changes in the statistical data, changes in the pattern or a sudden shift would further increase the unpredictability of the sequence, and these methods are unable to address this issue. The spatial non-stationarity indicates the distribution differences among different domains. Evidence showed that pre-training on diverse but weakly related datasets can degrade downstream performance compared to models trained directly on target data \cite{gongtowardsadap, xu2025gcb}. Only by dynamically adjusting the parameters of the distribution during cross-domain transfer can the performance of the model be improved.

\subsection{Novel Paradigms}
To address the issues mentioned above, some works have already begun to consider a paradigm shift.

\begin{itemize}[noitemsep, topsep=0pt, leftmargin=*]

\item \textbf{From Multimodal Forecasting to Causal Inference} To address the unpredictable impact of sudden changes in the external environment on time series, some recent studies have incorporated text as input, providing more environmental information to guide the direction of time series forecasting, such as Intervention-aware forecasting \cite{xu2024intervention}, From News to Forecast \cite{wang2024fromnews}, CiK \cite{williams2024cik}. To assess these emerging multimodal capabilities, high-fidelity multimodal benchmarks are being proposed as standardized evaluation frameworks \cite{xu2026fideltshighfidelitymultimodalbenchmark}. To put it more precisely, this is the potential application of causal inference in the field of time series, including some works such as the learning of causal structure \cite{chen2024causalstructure}, causal discovery \cite{cheng2023cuts}, and predictions in counterfactual scenarios \cite{yan2023counterfactual} \cite{wang2023counterfactual}.

\item \textbf{Test-Time Training (TTT)} Acknowledging the non-stationarity of time series and the impossibility for a fixed set of pre-trained weights to anticipate all future regime shifts, recent work has begun to adopt TTT as a principled alternative to the zero-shot assumption. TTT leverages information contained in the test instance itself \cite{sun2020ttt}. In temporal modeling, TTT has evolved from test-time adaptation, including general fields \cite{kim2024tttgeneral} \cite{kim2025tttgeneral} as well as specialized fields  \cite{guo2025ttttraffic}, to test-time learning \cite{christou2024test, xu2026internaldiagnosisexternalauditing}.

\item \textbf{Agents for Forecasting} The agent is an ideal environment that has the potential of integrating all of the aforementioned requirements. Existing agent frameworks use LLM to automatically perform feature engineering, model selection and code generation \cite{zhao2025timeseriesscientist, cai2025timeseriesgym, MoiraiAgentAgenticFramework}. However, such an AutoML process is far from sufficient. The agent still has not yet become the role responsible for conducting autonomous data retrieval, making decisions, and performing causal inferences \cite{xu2026contextualagenticmemorymemo}, presenting a critical bottleneck for continuous adaptation in time series.

\end{itemize}

\section{Argument I: The Category Error and Dimension Collapse}
The foundation of our critique lies in the distinction between a Modality and a Container. This distinction is not merely semantic; it dictates the success or failure of transfer learning, which is the engine of the Foundation Model paradigm.

\subsection{Modality vs. Data Type}
A \textbf{Modality}, such as Vision or Language, is governed by a group of invariant transformations (symmetries) that preserve semantics. In computer vision, concepts like translation invariance ($Cat(x, y) \approx Cat(x+\Delta, y)$) and occlusion are universal physical laws applicable to every image. These invariant laws create a structured manifold where transfer learning thrives: a model pre-trained on ImageNet learns to detect edges and textures that are genuinely useful when fine-tuned on medical X-rays \cite{zhang2022viability, xu2025clip}.

A \textbf{Container}, by contrast, is defined only by its data structure. Time series is a container defined by the tuple $(t, x_t)$; the semantic logic and generative rules governing each sequence are domain-specific and share no universal set of invariant transformations. 
\begin{itemize}[noitemsep, topsep=0pt, leftmargin=*]
\item \textbf{Physics (e.g., Weather, Energy):} Governed by conservation laws and differential equations. These systems are often continuous and exhibit smooth, cyclical dynamics \cite{chen2025foundation,dai_socnet_2025}.
\item \textbf{Economics (e.g., Finance, Sales):} Governed by game theory, psychology, and complex adaptive systems. These systems are often reflexive (predictions affect outcomes) and exhibit random walk behavior where history is a poor predictor of the future \cite{gottimukkala2025transfer_academic}.
\item \textbf{Biology (e.g., ECG, EEG):} Governed by electrochemistry and physiology. These signals are often strictly periodic but subject to catastrophic failure modes (arrhythmia) \cite{lee2025cognitive}.
\end{itemize}
Crucially, there are \textbf{no universally invalid transitions} in time series. A vision model can learn that an image of pure static noise is \textbf{invalid} because it mathematically violates the learned physics of light and objects. In time series, however, a system can undergo a legitimate \textbf{bifurcation}, transitioning from a stable, predictable state to a chaotic one in an instant. A monolithic foundation model therefore cannot dismiss a sudden structural break as noise or an error; it must accept it as a potential, legitimate reality.

The Category Error is not only limited to forecasting, but it also degrades discriminative tasks. Sequence dynamics carry contradictory semantics across domains. Consider an Anomaly Detection task where sensor values suddenly drop to zero and flatline. If this data comes from an industrial battery, it indicates a catastrophic sensor disconnection. However, if the same numerical sequence comes from retail sales, it simply indicates the store has closed for the night. Similarly, in Feature Extraction, a sudden localized spike might correspond to a promotional event in e-commerce, but to a thermal runaway in chemical engineering. If a generic TSFM maps these to the identical latent representation, it strips away semantics.

When disconnected domains are forced to share a single representation space, negative transfer frequently emerges. Evidence from recent benchmarks supports this. Research on popular foundational models shows that while they perform well on generic interpolation, they often struggle to beat simple univariate baselines like ARIMA or seasonal Naive methods on specific tasks \cite{ma2024assessing}. More importantly, studies on negative transfer in time series forecasting demonstrate that pre-training on a diverse corpus can systematically degrade performance on specific downstream tasks compared to training from scratch. This is particularly prevalent in finance, where the governing rules of the source domain (e.g., weather smoothness) actively contradict the target domain (e.g., market volatility) \cite{lee2024reinforced, ma2022intra}. This phenomenon has been empirically validated in studies showing that broadly pre-trained models struggle to outperform specialized baselines in financial risk analysis due to this exact domain impedance mismatch \cite{wu2025delphyne_arxiv}.

\subsection{The Modality Spectrum: Area-Specific vs. Universal Models}
We emphasize that the distinction between a Container and a Modality operates on a \textit{spectrum}. Some time series domains do share a substantial intersection of statistical symmetries, such as temporal smoothness, bounded variance, or strict seasonality. For instance, high-frequency energy load and traffic forecasting share similar periodic constraints. Co-training models on such intimately related clusters is theoretically sound because their structural intersection is non-trivial. We actively support the development of these \textbf{Area-Specific Large Models (ASLMs)}. Our theoretical critique is singularly targeted at the pursuit of \textbf{Universal TSFMs}—monolithic models that force disconnected domains (e.g., mixing the continuous differential equations of fluid dynamics with the reflexivity-driven random walks of finance) into a single latent space.

\subsection{A Formal Argument: Quotient Manifold Collapse}
We formulate the distinction between a Modality and a Container mathematically to analyze dimensionality reduction during representation learning. Let the ambient data space be $\mathcal{X} = \mathbb{R}^d$, where, for a time series of length $T$ and $C$ channels, $d = T \times C$.

\begin{definition}[Modality]
A \textbf{Modality} is a proper subset of $\mathcal{X}$, constrained by a specific, non-trivial invariance group $G_{\text{mod}}$. Data points $x$ are confined to a lower-dimensional quotient manifold $\mathcal{M} = \mathcal{X}/G_{\text{mod}}$ formed by the orbits of $G_{\text{mod}}$. Mathematically:
\begin{equation}
    \mathcal{M} = \{x \in \mathcal{X} \mid \text{Constraints}(x)\}, \quad \dim(\mathcal{M}) \ll d
\end{equation}
For a learning task $f: \mathcal{X} \to \mathcal{Y}$, this shared structural invariance implies that for all $g \in G_{mod}$, the function is invariant ($f(g \cdot x) = f(x)$) or equivariant. This allows the model to learn efficiently on the smaller quotient space.
\end{definition}

\begin{definition}[Container (Data Type)]
A \textbf{Container} refers to the entire space $\mathcal{X}$ itself, lacking any unifying, shared non-trivial symmetries across its subset distribution ($G_{\text{container}} = \{e\}$). 
\begin{equation}
    \text{Support}(D_{\text{container}}) \approx \mathcal{X}
\end{equation}
Consequently, its underlying topological quotient space collapses to an isomorphism with the original ambient space ($\mathcal{X}/\{e\} \cong \mathcal{X}$), meaning no global semantic dimensionality reduction is physically possible.
\end{definition}

\textbf{Theorem 3.1:} \textit{The universal space of all unconstrained time series domains $\mathcal{U} = \bigcup_{k} \mathcal{D}_k$ is a Container, and its corresponding universal quotient manifold dimension is $\dim(\mathcal{X})$.}

\textit{Proof} For $\mathcal{U}$ to function as a unified Modality, there must exist a universal symmetry group $G_{univ} \neq \{e\}$ that strictly preserves generative rules across every disparate constituent domain $\mathcal{D}_k \in \mathcal{U}$. Algebraically, $G_{univ}$ is defined as the intersection of all domain-specific stabilizer subgroups:
\begin{equation}
    G_{univ} = \bigcap_{k} \text{Stab}(\mathcal{D}_k)
\end{equation}
However, real-world temporal domains exhibit mutually exclusive, contradictory constraints. For instance, geometric scale transformations ($x \to c \cdot x$) inherently stabilize stochastic financial returns but completely violate strict amplitude constraints denoting physiological pathology in medical ECGs. 

Furthermore, control theory corroborates this impossibility. The Ho-Kalman algorithm for state-space realization proves that for \textit{any} mathematically arbitrary sequence vector, there exists a Linear Time-Invariant (LTI) system capable of perfectly generating it. Because any sequence is realizable by some underlying physical process, no universal structural properties uniquely rule out portions of $\mathbb{R}^d$. 

This absolute incompatibility across heterogeneous domains algebraically forces the intersection to collapse to the trivial identity element:
\begin{equation}
    G_{univ} \to \{e\}
\end{equation}
The effective dimension of the universal quotient manifold is thus strictly equal to the ambient space:
\begin{equation}
    \dim(\mathcal{X} / G_{univ}) = \dim(\mathcal{X} / \{e\}) = \dim(\mathcal{X})
\end{equation}
Since $\dim(\mathcal{M}_{univ}) = \dim(\mathcal{X})$, no functional semantic dimensionality reduction occurs. A representation cannot rely on shared semantics alone for $\mathcal{U}$. The learning algorithm is mathematically forced to strictly operate on the full high-dimensional ambient space $\mathbb{R}^{T \times C}$.

\subsection{The Mixture of Experts Defense and Observational Equivalence}
A significant counter-argument posits that modern architectures, specifically Mixture of Experts (MoE) models \cite{wu2026taskawaremixtureofexpertstimeseries}, learn the \textit{union} of disjoint manifolds ($\bigcup \mathcal{M}_k$), not their intersection. In this view, a gating network directs data to specialized expert sub-networks, bypassing the need for a universal grammar.

Our rebuttal to this relies on the principle of \textbf{Observational Equivalence} in short-context time series. In a modality like vision, a router easily distinguishes domains: a pixel grid of a cat has zero statistical overlap with a CT scan. In time series, raw numerical trajectories from disparate domains can be perfectly identical while being governed by contradicting generative laws. Consider a simple linear trajectory $x_{t-h:t} = [1, 2, 3, 4, 5]$. This five-point history is observationally equivalent under at least two incompatible generative processes:
\begin{enumerate}[noitemsep, topsep=0pt, leftmargin=*]
    \item \textbf{In Finance}, this pattern could represent a momentum signal within a stochastic process, governed by a random walk with a positive drift $c > 0$ and noise $\epsilon_t \sim \mathcal{N}(0, \sigma^2)$:
    \begin{equation}
    x_{t+1} = x_t + c + \epsilon_t
    \end{equation}
    The prediction is fundamentally uncertain, representing a continuation of a stochastic trend.
    \item \textbf{In Control Systems}, this same sequence could represent the deterministic ramp response of a first-order discrete-time linear system converging to a target setpoint $S=10$, controlled by a parameter $\alpha \in (0,1)$:
    \begin{equation}
    x_{t+1} = \alpha x_t + (1-\alpha)S
    \end{equation}
    The prediction represents a nonlinear convergence, slowing down as it reaches the peak.
\end{enumerate}

A router function $g(x_{t-h:t})$ observing only the numerical history faces an ill-posed inverse problem; it has no basis to prefer the stochastic over the deterministic explanation. To minimize global empirical risk across a varied dataset, the router's optimal mathematical strategy is to hedge, assigning non-zero probability to both experts. The final prediction becomes a weighted average of contradicting futures. Without explicit external semantic context, a universal MoE cannot route effectively; it simply averages.

\subsection{Sample Complexity Lower Bound}
The consequence of this trivial symmetry group is a fundamental barrier to efficient learning.

\begin{theorem}[Sample Complexity Lower Bound for Universal Approximation]
To learn a universal function class $\mathcal{F}$ over a data space treated as a \textbf{Container} (i.e., with a trivial shared symmetry group $G=\{e\}$), the number of samples $N$ required to achieve an error $\epsilon$ has a worst-case exponential dependency on the ambient dimension $d$:
\begin{equation}
    N \geq \Omega\left(\frac{1}{\epsilon^d}\right)
\end{equation}

\end{theorem}

\textit{Discussion.}
\textbf{1. Foundation Models for a Modality} (e.g., GPT-4, ResNet) succeed because they operate on a specific sub-manifold (e.g., $\mathcal{X}_{\text{text}}$) governed by shared symmetries (a grammar group for NLP; locality for vision). This shared bias constrains the learning problem, making it tractable and avoiding the curse of dimensionality.

\textbf{2. Foundation Models for a Container} (Time Series): When a model is trained on a union of disjoint domains with conflicting symmetries, it cannot exploit any single group structure for efficiency. Since $G_{\text{universal}} \approx \{e\}$, the model effectively faces the worst-case learning scenario of a Container. It must simultaneously learn incompatible generative processes, such as the Navier-Stokes equations for fluid dynamics, Ito Calculus for financial stochastic processes, or Electrophysiology for biological signals.

To make this more concrete, consider a \textbf{thought experiment} with a completely controllable electrical circuit. By designing a sufficiently complex input voltage $u(t)$, control theory guarantees that we can make the circuit's output voltage $y(t)$ follow \textit{any} desired waveform. A forecasting model that only observes the output time series $y(t)$ is blind to the external driving input $u(t)$. It is thus tasked with reverse-engineering a system capable of infinite behaviors from partial observations—an ill-posed and fundamentally unsolvable problem in the general case. Without architectural priors like weight sharing being consistently meaningful across domains, the model is forced towards \textbf{rote memorization}, leading to an over-parameterized and \textbf{non-generalizable} solution.

\subsection{Philosophical Design Implications: Exploring Modeling Trade-Offs}
The theoretical barrier resulting from dimension collapse directly challenges how we engineer predictive models. We synthesize three philosophical perspectives illuminating the systemic boundaries of attempting universal models versus parallel Area-Specific Large Models (ASLMs):

\begin{itemize}[noitemsep, topsep=0pt, leftmargin=*]
\item \textbf{No Free Lunch (NFL):} The NFL theorem loosely dictates that no single algorithm is universally superior \cite{wolpert1997no}. A monolithic model trained on disjoint temporal series learns a featureless prior because the disparate patterns fundamentally cancel each other out—financial volatility nullifies weather smoothness. 
\item \textbf{Occam's Razor:} The foundation model paradigm inherently assumes that an aggregated global architecture is necessarily superior to robust domain-specific solutions. However, assuming completely distinct operational physics (like EKG impulses and market fluctuations) share a cohesive generalized sequence rule violates the principle of parsimony. It introduces unproven, heavy architectural complexities where simpler structures already excel.
\item \textbf{The Pareto Frontier:} When balancing parameter capacity and data size against forecasting accuracy, an unconstrained single foundation model faces severe performance ceilings. Expanding a single model to capture contradictory domain structures yields aggressive diminishing returns. Conversely, deploying multiple parallel, moderately-sized Area-Specific Large Models (ASLMs) directly improves computational resource utilization while averting negative transfer on the Pareto boundary \cite{wu2026taskawaremixtureofexpertstimeseries}.
\end{itemize}

\section{Argument II: Non-Stationarity as the Consequence of Category Error}
The theoretical trifecta detailed in Argument I—a trivial shared symmetry group, exponential sample complexity, and intrinsic incompressibility—forces a critical compromise when constructing a universal foundation model. Without the structural grounding of a true modality, the model loses the ability to maintain domain-specific \textit{inductive biases}. We argue that the crippling inability of foundation models to handle temporal \textbf{non-stationarity} is not an orthogonal issue, but the direct, malignant consequence of this Category Error.

\subsection{An Open Hypothesis: The ``Generic Filter" Degeneration}
In the absence of strong, domain-specific inductive biases, a foundation model trained on disjoint domains is forced to find the lowest common denominator of temporal dynamics to minimize global empirical risk. This optimization path of least resistance potentially causes the model to degenerate into an expensive, parameter-heavy implementation of classical signal processing kernels. Recent investigations into attention collapse in time series Transformers support this, showing that self-attention mechanisms often collapse into local window attention or periodic attention \cite{dong2023stabilizing}. We formally pose the \textbf{``Generic Filter Degeneration''} not as an isolated failure, but as an open hypothesis and critical vulnerability requiring community-wide investigation. Does this attention collapse invariably occur under massive cross-domain pre-training? Determining which specific architectures are most susceptible and identifying the distinct dataset features that trigger this semantic collapse offers a crucial future research direction for the community.

This degeneration is further exacerbated by the prevailing evaluation frameworks. The adoption of standard point-wise metrics (e.g., MSE) introduces a ``bad inductive bias'' that fundamentally favors smooth outputs \cite{guenShapeTimeDistortion2019,qiuDBLossDecompositionbasedLoss2025}. The model mathematically minimizes its expected risk by becoming a smooth filter \cite{wangFrequencyMattersWhen2025}, cementing its role as an inflexible moving average. If a billion-parameter Transformer effectively functions as a moving average, we have achieved a massive regression in computational efficiency. This hypothesis clarifies why simple linear models like DLinear consistently match or outperform complex Transformers \cite{zeng2023are, zeng2025position}. 

\subsection{Vulnerability to Non-Stationarity and Interventions}
This forced homogenization inherently washes out domain-specific dynamics, rendering the universal model uniquely brittle to within-series non-stationarity. Current forecasting paradigms are largely autoregressive: they predict the future $X_{t+1}$ based solely on the numerical history $X_h = X_{t-h:t}$. However, real-world systems are open, subject to external interventions ($U_t$)—such as a sudden market crash or a policy change—that are unobserved in the scalar time series history. This fragility is mathematically formalized by the Autoregressive Blindness Bound \cite{xu2024intervention}, establishing an intrinsic mathematical limit of history-based forecasting in non-stationary environments.

\textbf{Proposition 3.2 (Autoregressive Blindness Bound):} \textit{In a dynamical system driven by external interventions, predicting true intervention-driven system dynamics $F(X_h, U_t)$ using any model that relies solely on historical state variables $X_h$ will have a strict, non-zero lower bound on its prediction error. This irreducible error floor is directly proportional to the magnitude and variance of the unobserved external intervention $U_t$ \cite{xu2024intervention}.}

Because foundation models degenerate into these history-only smooth filters, they are critically blind to external interventions. What is casually termed ``Concept Drift" is often just this Unobserved Intervention. Consequently, the model is not just ignorant of regime shifts; it is confidently wrong, projecting past smoothed biases onto new, disrupted regimes. In high-stakes environments, relying on such inflexible filters is fundamentally unsafe.

\section{Addressing Alternative Views}
\subsection{The Success of TSFMs: Memorization and Metric Alignment}
\textbf{The Counter-Argument:} Proponents naturally point out that Universal Time Series Foundation Models (TSFMs) like Chronos and Moirai are demonstrating increasingly impressive ``Zero-Shot'' performance on modern comprehensive leaderboards such as GIFT-Eval \cite{aksu2024giftevalbenchmarkgeneraltime} and fev-bench \cite{shchur2025fevbenchrealisticbenchmarktime}. They argue that emerging ``Neural Scaling Laws" \cite{nguyen2024towards_openreview} prove that scale allows these models to ``grok" universal rules.

\textbf{Our Rebuttal:} We fully acknowledge the impressive empirical interpolation capabilities of state-of-the-art TSFMs on these static leaderboards. Our core stance is not to blindly deny their achievements, but to encourage deep reflection on \textit{why} they appear to work. Recent analyses reveal two severe underlying vulnerabilities driving this success:
\begin{enumerate}[noitemsep, topsep=0pt, leftmargin=*]
    \item \textbf{Scaling Accelerates Memorization, Not Generalization:} As scaling laws take effect, in-distribution (ID) accuracy improves substantially, but out-of-distribution (OOD) generalization actively degrades \cite{nguyen2024towards_openreview}. This suggests massive TSFMs act as efficient ``lookup tables" that memorize historical containers, rather than models that extract true generative rules capable of adapting to novel OOD shocks.
    \item \textbf{Data Leakage and Inductive Bias in Metrics:} The dominance on comprehensive benchmarks is heavily padded by unmeasured information leakage \cite{meyerRethinkingEvaluationEra2026}. Furthermore, modern probabilistic indicators like the Continuous Ranked Probability Score (CRPS) are essentially aggregations of point-wise errors across quantile heads. Because both standard MSE and CRPS penalize variance and favor smooth outputs \cite{guenShapeTimeDistortion2019,qiuDBLossDecompositionbasedLoss2025}, scaling these models simply creates a homogenized posterior. This creates a self-fulfilling ``illusion of scaling'' that masks the model's inability to capture true high-frequency domain intelligence \cite{wangFrequencyMattersWhen2025}.
\end{enumerate}

\subsection{The ``Zero-Shot" Concession: Initialization vs. Forecasting}
\textbf{The Counter-Argument:} The primary value of models like Chronos is their ``Zero-Shot" capability, democratizing forecasting \cite{ansari2025chronos}.

\textbf{Our Rebuttal:} We agree that FMs are powerful tools for initializing a forecast in Cold Start problems. A generic prior is undoubtedly better than a random guess. However, a substantial portion of high-stakes industrial environments operate in ``Warm Start" scenarios, where continuous sensor streams exist. For example, predicting industrial battery degradation relies entirely on real-time assimilation of early historical cycles \cite{seversonDatadrivenPredictionBattery2019b, dai_socgate_2025}, and modern physical system forecasting (e.g., weather) strictly requires dynamic feedback mechanisms to adjust to continuous distribution shifts \cite{biAccurateMediumrangeGlobal2023,kochkovNeuralGeneralCirculation2024}. In these cases, a lightweight model trained Just-in-Time on the specific series can outperform in warm-start regimes a static zero-shot FM.

\subsection{The Agentic Forecasting Paradigm and Correlation Limits}
\textbf{The Counter-Argument:} LLMs have the potential to become the central hub for time series analysis, fulfilling the role of core predictors by performing direct forecasting \cite{jin2024position}.

\textbf{Our Rebuttal:} We fundamentally agree with using LLMs as central orchestrators, but we reject their role as raw continuous Numerical Predictors. LLMs excel at modeling high-level semantic correlations, but they are not sufficient at learning true causal or generative temporal structures solely from observational data \cite{jinCLadderAssessingCausal2023}. Aligning a chaotic financial series or a deterministic fluid dynamic simulation directly with textual tokens is a mathematically lossy compression. An LLM possesses ``world knowledge" (e.g., recognizing that a holiday drives retail sales), but this is contextual correlation, not intrinsic sequence generation. Therefore, an LLM should operate explicitly outside the numerical loop—reasoning about external interventions and routing control signals—rather than forcefully ingesting normalized numerical arrays into its token vocabulary.

\section{Call to Action: Prioritizing Time-to-Recovery}
The ``Foundation Model" paradigm is frequently applied uniformly across domains without acknowledging their underlying operational differences. In time series, the underlying generative rules are shifting, shape-changing, and subject to external interventions. A static model is a fundamentally flawed tool in such open environments. To steer the field toward more robust solutions, we call for a paradigmatic shift in how the community evaluates and builds forecasting systems.

\subsection{The Core Metric: Time-to-Recovery (TTR)}
We must pivot our evaluation priorities away from static ``Zero-Shot" performance on fixed benchmarks like M4 \cite{makridakis2022m4}, GIFT-Eval \cite{aksu2024giftevalbenchmarkgeneraltime}, and fev-bench \cite{shchur2025fevbenchrealisticbenchmarktime}. Reversals and chaotic shifts are inevitable in real-world systems. The most critical operational question is not the error during artificially stable periods, but rather: \textit{``How many data points does it take for a model to return to a baseline error after a structural break is introduced?"} 

We strongly urge the community to adopt \textbf{Time-to-Recovery (TTR)} as the primary metric for adaptability. To define TTR formally, let $t_I$ denote the specific time index at which a structural break (an external intervention) occurs. Let $E(t)$ denote the point-wise prediction error metric evaluated at a given time index $t$. Let $E_{\text{baseline}}$ denote the model's average historical error evaluated over a strictly stable period prior to the intervention. Finally, let $k \in \mathbb{Z}^+$ act as the step counter for elapsed time post-intervention. TTR is defined as the minimum number of time steps required for a model's error to recover to its steady pre-break baseline:
\begin{equation}
        \text{TTR} = \min \{ k \in \mathbb{Z}^+ \mid E(t_I + k) \le E_{\text{baseline}} \}
\end{equation}
This metric mechanically penalizes static, history-only models that confidently project obsolete smoothed priors into a new regime, and rewards dynamic systems capable of agile adaptation.

\subsection{A Pathway to TTR: The Conditional Control Architecture}
How do we practically achieve a minimal TTR? We suggest migrating from pure \textit{Multimodal Fusion} (forcing all data into one giant numerical tokenizer) toward \textbf{Conditional Control}. Conceptually, this paradigm admits a diverse variety of architectural implementations, but effectively relies on the explicit separation of perception from reasoning via three core interacting modules:

\begin{enumerate}[noitemsep, topsep=0pt, leftmargin=*]
    \item \textbf{Perceiver (Encoder):} A dedicated encoder tasked with processing external, often unstructured contextual data to extract an actionable and formalized intervention signal $U_t$.
    \item \textbf{Controller (Reasoning Core):} A routing logic layer that determines the specific system regime based on the detected $U_t$. It acts as a heuristic switch to evaluate whether the intervention warrants overriding the default predictive solver with an alternative modeling path.
    \item \textbf{Solver (Actuator):} A dynamic predictor that executes the final forecast strictly conditional on the exogenous factor $U_t$. Its design must ensure adaptability to non-stationary structural shifts without invoking catastrophic inference latency.
\end{enumerate}

Detailed architectural instantiations and practical implementation methodologies of this conceptual framework—including concrete examples of exactly which frameworks fulfill each functional module across varying structural designs—are extensively explored in Appendix~\ref{sec:appendix_architecture}. The solver inherently bypasses the Autoregressive Blindness Bound: it no longer attempts to perfectly predict an external shock from a scalar history, but merely conditionally responds to the $U_t$ control signal. By acknowledging that algorithmic intervention-awareness supersedes blind monolithic parameters, we can build inherently robust, adaptive forecasting systems.

\section{Conclusion}
Time series is not a language. It represents the continuous generative dynamics of physical and socioeconomic systems. It requires conditionally responsive architectures, not just a massive static dictionary. The future of robust forecasting lies not simply in increasing model capacity, but in developing adaptable control systems capable of immediate recovery.

\newpage
\bibliographystyle{plainnat}
\bibliography{reference}


\clearpage
\newpage
\appendix

\section{Further Discussion on Alternative Viewpoints}

\subsection{Is the Controller a Universal Model in Disguise?}
\label{sec:appendix_controller}
A valid critique of our proposed Causal Control System is to question whether the Controller module itself becomes a universal model, simply pushing the problem up one level of abstraction. If the Controller must understand all possible regime shifts to select the right solver, does it not face the same generalization challenge?

The crucial difference lies in the nature and complexity of the problem being solved. A universal foundation model for forecasting is tasked with a high-dimensional, continuous \textit{trajectory generation} problem. In contrast, our Controller is tasked with a far simpler, often low-dimensional \textit{regime classification} or decision problem. Its goal is not to predict the exact future path of the series, but to answer a simpler question: ``Based on the intervention signal $U_t$, should I use the default solver or switch to an alternative?'' This classification task is fundamentally less complex and less susceptible to the curse of dimensionality than generating the forecast itself.

\section{Further Discussion on Theoretical Details}

\subsection{Practical Considerations of JIT Solvers}
\label{sec:appendix_jit}
The proposal to train a solver Just-in-Time (JIT) raises practical concerns about latency and stability, which we address here.

\textit{Latency.} The assumption that JIT training is slow is a misconception when comparing it to the correct baseline. The computational cost of training a lightweight linear model or a small gradient boosting tree on a recent window of a few hundred data points is often measured in milliseconds on a CPU. This is frequently faster and more resource efficient than performing a single forward pass on a billion-parameter foundation model, which typically requires specialized hardware like GPUs and can have significant inference latency.

\textit{Stability.} A JIT solver trained on a small, recent window of data could indeed exhibit high variance and overfit to local noise. A robust implementation of our paradigm would not be a binary switch between a global model and a purely local one. Instead, a more sophisticated approach would involve using a global prior, perhaps from an Area-Specific Foundation Model (ASFM), to regularize the JIT solver. This can be framed as a Bayesian update, where the JIT solver fine-tunes the global prior based on the most recent evidence, providing a balance between stability and adaptability.

\subsection{Beyond Tautology: The Role of the Autoregressive Blindness Bound}
\label{sec:appendix_blindness}
One could argue that the Autoregressive Blindness Bound is tautological: if a model does not observe a variable, it cannot account for its effects. While true, the value of the bound is not in this trivial statement. Its power lies in formally refuting the \textit{implicit belief} embedded within the universal foundation model paradigm. This belief is that with sufficient scale and data, a model can learn a universal prior that acts as a proxy for, or learns the distribution of, all possible future interventions. The bound demonstrates that this is not possible for unique, out-of-distribution shocks. It exposes the failure of this implicit assumption, proving that no amount of historical data can substitute for the direct observation of the causal mechanism of an intervention.

\subsection{Further Directions for Benchmarking}
\label{sec:appendix_benchmarking}
In Section 7, we advocate for Time-to-Recovery (TTR) as a primary metric for evaluating model adaptability. To complement this, we propose an additional metric: \textbf{Intervention Recall}. This metric would measure whether a system correctly identifies that a structural break has occurred. For our Causal Control System, this would specifically test the Controller's ability to detect a meaningful change from the intervention signal $U_t$.

Formally, let $\mathcal{I} = \{t_1, t_2, \ldots, t_m\}$ be the set of discrete time steps where true structural breaks or interventions occur. Let $\mathcal{D}(t)$ be a binary detection function from the model, where $\mathcal{D}(t) = 1$ if the model signals an intervention at or near time $t$, and $0$ otherwise. Intervention Recall can be defined as:
\begin{equation}
\text{Recall}_{\text{intervention}} = \frac{\sum_{t_i \in \mathcal{I}} \mathbf{1}(\exists t' \in [t_i, t_i+\delta] \text{ s.t. } \mathcal{D}(t')=1)}{|\mathcal{I}|}
\end{equation}
where $\mathbf{1}(\cdot)$ is the indicator function, and $\delta$ is a small, predefined tolerance window to account for minor detection delays. This metric quantifies the fraction of true interventions that the model successfully detected.

Measuring Intervention Recall is crucial because it decouples a model's ability to \textit{detect} a regime change from its ability to \textit{adapt} to it. A high TTR score might be achieved by a model that is inherently volatile and recovers from shocks by chance, not by understanding. By evaluating both Intervention Recall and TTR, we can create a more holistic picture of a model's robustness, rewarding systems that both correctly identify and efficiently adapt to structural changes in the underlying generative process.

\section{Practical Implementation of the Conditional Control Architecture}
\label{sec:appendix_architecture}

Translating the conceptual Conditional Control paradigm into an industrial-grade system permits multiple flexible architectural instantiations. While all adhere to the tripartite structure (Perceiver, Controller, Solver), the specific realization of each component strictly dictates the system's performance ceiling, floor, required computational resources, and applicable scenarios. We outline four representative implementations below.

\subsection{Implementation I: The LLM Agent-Centric Framework}
\textbf{Perceiver Component:} An LLM operating as a prompt interpreter, extracting qualitative context from exogenous text. \textbf{Controller Component:} The LLM Agent orchestrator utilizing extensive historical context and real-time logs. \textbf{Solver Component:} An isolated code-execution sandbox or integrated web-search module.  

\textit{Mechanism:} The Controller dynamically writes, executes, and iteratively debugs custom statistical code within the sandbox based on reasoning over the history, outputting final predictions directly.  

\textit{Analysis:} This architecture boasts a massive performance upper bound (ceiling) due to the sheer generative versatility of the LLMs and search engines. However, its lower bound (floor) is unstable, heavily dependent on the LLM's stochastic prompting. It strictly requires immense computational resources and is exclusively suited for low-frequency, highly complex macroeconomic or strategic planning scenarios.

\subsection{Implementation II: The Lightweight Neuromorphic Controller}
\textbf{Perceiver Component:} Numerical sensors tracking real-time error trajectories from historical targets. \textbf{Controller Component:} A lightweight control-theoretic module (e.g., Kalman Filters) or a small specialized neural network acting as an error-correction router. \textbf{Solver Component:} A massive, frozen Time Series Foundation Model (TSFM) or an LLM.  

\textit{Mechanism:} This structure acts inversely to Implementation I. The Controller uses strict control-theoretic formulas over the true past values and recent prediction errors to continuously orchestrate, recalibrate, or mathematically correct the reasoning outputs derived from the heavy TSFM/LLM solver (e.g., online Mixture of Experts).  

\textit{Analysis:} This guarantees high stability (a strong lower bound). It minimizes online compute and is ideal for scenarios requiring highly reliable, continuous online adaptation where FMs are known to hallucinate sequence continuations.

\subsection{Implementation III: The Multimodal Encoder Supervisor}
\textbf{Perceiver Component:} A true multimodal encoder network ingesting diverse external priors (e.g., Vision Transformers for images, discrete embedding layers for text). \textbf{Controller Component:} A dedicated fusion neural network. \textbf{Solver Component:} A numerical downstream forecaster.  

\textit{Mechanism:} The Controller actively fuses the extracted multimodal indices (e.g., visual cloud cover layout or textual policy shifts) directly into the downstream Solver's latent space, conditioning its numeric predictions via Cross-Attention layers.  

\textit{Analysis:} Highly effective in specialized fields where numerical series are inextricably tied to qualitative environments (e.g., weather forecasting, medical diagnostics). It trades generative versatility for domain-specific precision, possessing a high developmental cost for the specialized multimodal alignment.

\subsection{Implementation IV: LLM Agent alongside JIT Lightweight Solvers}
\textbf{Perceiver Component:} Feature extraction modules monitoring concept drift signals. \textbf{Controller Component:} An automated workflow routing manager guided by LLM heuristics. \textbf{Solver Component:} Variable, ranging from Just-in-Time (JIT) shallow linear layers to deeper autoregressive Transformers.  

\textit{Mechanism:} The Controller analyzes the required latency and data environment. In fast, data-sparse, or highly volatile scenarios (e.g., high-frequency industrial monitoring and quantitative trading), it spins up shallow, lightning-fast JIT layers perfectly adapted on a short window. In slower, data-rich scenarios, it conditionally delegates the task to a deeper, pre-compiled Transformer.  

\textit{Analysis:} This is the premier industrial architecture due to its ``dual-path" routing capability. It perfectly balances theoretical latency (computational resource efficiency) with modeling depth, effectively serving both edge IoT reflex requirements and intensive centralized analytical needs.


\end{document}